\documentclass[twoside]{article}

\usepackage[accepted]{aistats2024}
\usepackage[round]{natbib}

\usepackage[utf8]{inputenc} 
\usepackage[T1]{fontenc}
\usepackage{url}              
\usepackage{cite}             
\usepackage{hyperref}

\usepackage{algorithm}
\usepackage{algorithmic}
\usepackage{fancyhdr}

\usepackage{graphicx}         
\usepackage{booktabs}         

\usepackage{amssymb,amsfonts,amsthm,bm,bbm,dsfont}
\newtheorem{theorem}{Theorem}
\newtheorem{remark}{Remark}

\newtheorem{lemma}{Lemma}
\newtheorem{definition}{Definition}

\usepackage{comment}

\usepackage{soul}

\DeclareSymbolFont{bbold}{U}{bbold}{m}{n}
\DeclareSymbolFontAlphabet{\mathbbold}{bbold}

\usepackage{color}

\newcommand{\cE}{\mathcal{E}}
\newcommand{\cF}{\mathcal{F}}

\newcommand{\cL}{\mathcal{L}}

\newcommand{\cR}{\mathcal{R}}
\newcommand{\cS}{\mathcal{S}}

\newcommand{\cX}{\mathcal{X}}
\newcommand{\cY}{\mathcal{Y}}
\newcommand{\cZ}{\mathcal{Z}}

\newcommand{\boldh}{\mathbf{h}}

\newcommand{\boldq}{\mathbf{q}}

\newcommand{\boldv}{\mathbf{v}}
\newcommand{\boldw}{\mathbf{w}}
\newcommand{\boldx}{\mathbf{x}}

\newcommand{\x}{\mathbf{x}}
\newcommand{\D}{\mathcal{D}(\mathcal{Y})}

\newcommand{\kl}{\mathsf{KL}}
\renewcommand{\boldsymbol}[1]{\mathbf{#1}}
\allowdisplaybreaks
\begin{document}

%

%

\twocolumn[

\aistatstitle{Online Distribution Learning with Local Private Constraints}

\aistatsauthor{ Jin Sima$^*$ \And Changlong Wu$^*$ \And Olgica Milenkovic \And Wojciech Szpankowski
}

\aistatsaddress{UIUC \And  Purdue University \And UIUC \And Purdue University}]

\begin{abstract}
  We study the problem of online conditional distribution estimation with \emph{unbounded} label sets under local differential privacy. Let $\mathcal{F}$ 
  be a distribution-valued function class with unbounded label set.
  We aim at estimating  an \emph{unknown} function $f\in \mathcal{F}$ in an online fashion so that at time $t$ when the context $\x_t$ is provided we can generate an estimate 
  of $f(\x_t)$ under KL-divergence  
  knowing only a \emph{privatized} version of the true \emph{labels} sampling from $f(\x_t)$. The ultimate objective  is to minimize the cumulative KL-risk of a finite horizon $T$. We show that under $(\epsilon,0)$-local differential privacy of the privatized labels, the KL-risk grows as $\tilde{\Theta}(\frac{1}{\epsilon}\sqrt{KT})$ upto poly-logarithmic factors where $K=|\mathcal{F}|$. This is in stark contrast to the $\tilde{\Theta}(\sqrt{T\log K})$ bound demonstrated by~\citet{wu2023learning} for \emph{bounded} label sets. As a byproduct, our results  recover a nearly tight upper bound for the hypothesis selection problem of~\citet{gopi2020locally} established only for the \emph{batch} setting.
\end{abstract}

\section{Introduction}

Online conditional distribution learning (a.k.a., sequential probability assignments)~\citep{shtarkov87,xb97,merhav95,lugosi-book,bilodeau2021minimax} is a fundamental problem that arises in many different application contexts, including universal source coding~\citep{xb97,merhav95,ds04,sw12}, portfolio optimization~\citep{lugosi-book}, and more recently, the interactive decision making~\citep{foster2021statistical}. Let $\mathcal{X}$ be a context (or features), $\mathcal{Y}$ be a set of labels, and $\mathcal{D}(\mathcal{Y})$ be the set of all probability distributions over $\mathcal{Y}$. For any distribution-valued class $\mathcal{F}\subset \mathcal{D}(\mathcal{Y})^{\mathcal{X}}$, the problem is formulated as a game between \emph{Nature} and \emph{learner} with the following protocol: at beginning Nature selects $f\in\mathcal{F}$; at each time step $t$ Nature selects a context $\x_t$; the learner then predicts $\hat{p}_t\in\mathcal{D}(\mathcal{Y})$ for the distribution of the next label; Nature then generates  $p_t=f(\x_t)$, samples $y_t\sim p_t$ and reveals label $y_t$ to the learner. The goal is to minimize the cumulative risk
$\sum_{t=1}^T\mathsf{D}(f(\x_t),\hat{p}_t)$, for appropriately selected divergence function $\mathsf{D}$.

It can be shown that for \emph{any} finite class $\mathcal{F}$ the cumulative risk is upper bounded by $\log|\mathcal{F}|$ if we take $\mathsf{D}$ to be the KL-divergence. There are many other work that extend this basic setup to different formulations, including non-parametric infinite classes, miss-specified setting, stochastic feature generation scenarios and computational efficient methods~\citep{rakhlin2015sequential,bhatt2021sequential,bilodeau2021minimax,bilodeau2020relaxing,wu2022expected, bhatt2023smoothed}.

This paper investigates a different angle for this fundamental problem in which  data are shared \emph{privately} to the learner. We employ the concept of \emph{local differential privacy} (LDP) for privatizing the \emph{labels} $y_t$s shared to the learner. Formally, we consider the following game between \emph{three} parties, Nature, learner and the \emph{users}: (i) at beginning Nature selects $f\in \mathcal{F}$; (ii) Nature at each time step $t$  selects $\x_t$ and reveals to learner; learner makes prediction $\hat{p}_t\in \mathcal{D}(\mathcal{Y})$; (iii) the \emph{user} then samples $y_t\sim f(\x_t)$ but reveals a privatized version $\Tilde{Y}_t$ to the learner. Our main focus is on minimizing  the following KL-risk $\sum_{t=1}^T\mathsf{KL}(f(\x_t),\hat{p}_t)$ subject to local differential privacy constrains on the $\tilde{Y}_t$s, where $\mathsf{KL}$ is the KL-divergence.

\subsection{Results and Techniques}
Our main results of this paper establish nearly matching  lower and upper bounds for the KL-risk of our private online conditional distribution estimation problem when the label set $\mathcal{Y}$ are \emph{unbounded}.

\begin{theorem}[Lower Bound]
    There exists a finite class $\mathcal{F}$ of size $K$ with $|\mathcal{Y}|\le K$ such that for any $(\epsilon,0)$-local differential private mechanism and learning rules, the KL-risk is lower bounded by $\Omega(\frac{1}{\epsilon}\sqrt{KT})$. 
\end{theorem}

A pertinent setup was studied recently by~\citet{wu2023learning}, where the authors demonstrated that the KL-risk can be upper bounded by $\Tilde{O}(\sqrt{T\log K})$ via a randomized response mechanism. However, it was assumed implicitly that the alphabet size of $\mathcal{Y}$ is a \emph{constant} when defining their noisy kernel in a differentially private manner. Our lower bound demonstrates that such an logarithmic dependency on the size $K$ cannot be achieved for \emph{unbounded} label set $\mathcal{Y}$.

Our next main result shows that the lower bound is essentially tight for \emph{unbounded} label sets, upto poly-logarithmic factors:
\begin{theorem}[Upper Bound]
        For \emph{any} finite class $\mathcal{F}$ of size $K$ with \emph{any} $\mathcal{Y}$, there exists $(\epsilon,0)$-local private differential private mechanism and learning rules, such that with \emph{adversarially} generated features, the KL-risk is upper bounded by $O(\frac{1}{\epsilon}\sqrt{KT\log^5(KT)})$.
\end{theorem}
Our main technique for achieving this upper bound is based on a modified version of the EXP3 algorithm by appropriately defining the loss vector using the \emph{log-likelihoods} of distributions specified by $\mathcal{F}$. Unlike the conventional randomized response that perturbs the labels, we add noise directly to the log-likelihood vectors at each time step. Note that the main challenge here is that the log-likelihoods are generally \emph{unbounded}. Moreover, we need the scale of (Laplacian) noise of order $K/\epsilon$ to achieving $(\epsilon,0)$-differential privacy for a vector of dimension $K$. To resolve the first issue, we employ a "clipping" technique as~\citet{gopi2020locally} for controlling the log-likelihoods, but it needs careful adaption that is suited for bounding our $\kl$-risk. To resolve the second issue, we employ a novel randomized approach that reveals only one component of the log-likelihood vector at each time step which reduces the need of noise scale from $K/\epsilon$ to $O(1/\epsilon)$. We then apply an adaption of the regret analysis of EXP3 algorithm for deriving our desired KL-risk bound.

Finally, by applying the \emph{online-to-batch} conversion technique and Pinsker's inequality, we show that our upper bound for the KL-risk also implies a nearly tight upper bound for the \emph{batch} setup of~\citet{gopi2020locally} with \emph{non-interactive} privatization mechanisms.

\vspace{-0.1in}
\subsection{Related Work}
Our setup is related to the locally differential private hypothesis selection problem, as introduced in~\citet{gopi2020locally}. This can be seen as a \emph{batch} version of our setup  with the \emph{constant} function class. It is crucial to emphasize that our online setup poses substantial technical challenges. This is primarily because the underlying distributions, which we are attempting to learn, are \emph{changing} at every time step and are unknown \emph{a-priori}. Unlike the local privacy studied in this paper, online learning with \emph{central} differential privacy was studied extensively in the literature, see~\citet{dwork2014algorithmic, golowich2021littlestone, kaplan2023differentially, asi2023private}. Learning with locally private labels for \emph{batch} case and with different loss functions was discussed in~\citet{chaudhuri2011sample,ghazi2021deep,wu2022does,esfandiari2022label}. Our lower bound proof is also related in spirit to those employed in~\citet{ullman2018tight}.

\paragraph{Summary of contributions.} Our main contributions can be summarized as follows: (i) we formulate a novel {\it online} distribution learning problem with {\it local} differential privacy constraints; (ii) we demonstrate a surprising lower bound for the KL-risk of $\Omega(\sqrt{K T})$, in contrast to the $\tilde{O}(\sqrt{T\log K})$ upper bound known in literature only for \emph{bounded} label sets; (iii) we present an novel privatization mechanism and learning rule based on EXP3 that nearly matches the lower bound.

\section{Problem Setup and Preliminaries}

In this paper, we are interested in the \emph{locally} differential private setting where each entry $y_t\in\cY$, $t\in[T]$, is protected by a message $\Tilde{y}_t\in\Tilde{\cY}$, generated by some randomized mapping $A_t:\cY\rightarrow\Tilde{\cY}$ that is $(\epsilon,\delta)$-differentially private. Here, we do not impose any constraints on the choices $\cY$ and $\Tilde{\cY}$ for the generality of the definition of differential privacy. We specify $\cY$ and $\Tilde{\cY}$ in the following section (Section \ref{sec:problemform}).  
\begin{definition}\label{def:ldp}
A privatization scheme $(A_1,\ldots,A_T)$ where $A_t:\cY\rightarrow\Tilde{\cY}$, $t\in[T]$, is $(\epsilon,\delta)$-locally differentially private for some $\epsilon,\delta>0$ if for any different $y_t,y'_t\in\cY$ and $S_t\subseteq\Tilde{\cY}$, $t\in[T]$,
\begin{align}\label{eq:ldp}
    \mathrm{Pr}(A_t(y_t)\in S_t)\le e^{\epsilon}\mathrm{Pr}(A_t(y'_t)\in S_t)+\delta. 
\end{align}
If $\delta=0$, the privatization scheme $(A_1,\ldots,A_T)$ is called pure $\epsilon$-locally deferentially private. 
\end{definition}
In this paper, we consider only the case when $A_t$ are independent for different $t\in[T]$, i.e., each $A_t$ uses only private coins.

\subsection{Problem Formulation}\label{sec:problemform}

Let $\cX=\mathbbm{R}^d$ be the feature space for some positive integer $d$ and $\cY=[M]=\{1,\ldots,M\}$ be the label space for some positive integer $M$. It is assumed that $M$ and $d$ can take any  value. Denote $\D=\{(u_1,\ldots,u_M)\in\mathbbm{R}^M:\sum^{M}_{i=1}u_i=1,u_i\ge 0,i\in[M]\}$ as the set of probability distributions over $\cY$.
Let $\cF=\{f_1,\ldots,f_K\}$ be a hypothesis set where $f_j:\cX\rightarrow \D$, $j\in[K]$ is a function that maps a feature $\boldx\in\cX$ to a distribution $f_j(\boldx)\in \D$ over the label space. For any $y\in\cY$, $j\in[K]$, and $\boldx\in\cX$, let $f_j(\boldx)[y]$ be the probability mass of element $y$ in distribution $f_j(\boldx)$.

Consider an online game between \emph{Nature}, \emph{Learner} and the \emph{Users} 
where Nature arbitrarily picks a function $f\in \cF$ at time $t=0$ and the Learner wishes to learn $f$ over a time period $t\in[T]$ using privatized data generated by the Users. At each time $t\in [T]$, Nature arbitrarily picks a feature $\boldx_t\in\cX$ and reveals $\boldx_t$ to the Learner. The Learner then makes an estimate $\hat{p}_t=\Phi_t(\boldx^t=(\boldx_1,\ldots,\boldx_t),\tilde{\boldsymbol{Y}}^{t-1}=(\tilde{\boldsymbol{Y}}_1,\ldots,\tilde{\boldsymbol{Y}}_{t-1}))$ of $f(\boldx_t)$ based on the feature history $\boldx^t$, privatized label history $\tilde{\boldsymbol{Y}}^{t-1}$, and a function $\Phi_t:\cX^t\times\Tilde{\cY}^{t-1}\rightarrow \D$. After the Learner makes the estimate $\hat{p}_t$, the User then samples a label $Y_t$ independently according to the distribution $f(\boldx_t)$, generates (independently for different $t\in[T]$) a privatized version $\tilde{\boldsymbol{Y}}_t=\cR_t(Y_t)\in \mathbbm{R}^K$ of $Y_t$ using a random mapping $\cR_t:\cY\rightarrow \mathbbm{R}^K$, and reveals $\tilde{\boldsymbol{Y}}_t$ to the Learner. It is required that the privatization scheme $\cR^T=(\cR_1,\ldots,\cR_T)$ is $(\epsilon,\delta)$-locally differentially private. Hence, 
\begin{align}\label{eq:ldponline}
        \mathrm{Pr}(\Tilde{\boldsymbol{Y}}_t\in S|Y_t)\le e^{\epsilon}\mathrm{Pr}(\Tilde{\boldsymbol{Y}}_t\in S|Y'_t)+\delta  
\end{align}
for any $S\subseteq\mathbbm{R}^K$, $Y_t,Y'_t\in\cY$, and $t\in[T]$. Here, we choose $\Tilde{\cY}=\mathbbm{R}^K$ for convenience of the construction of our privatization schemes as introduced in the following sections. In general, $\Tilde{\cY}$ can be arbitrarily chosen.

Our problem is to design an $(\epsilon,\delta)$-locally differentially private  scheme $\cR^T$ for the Nature and a prediction scheme $\Phi^T$ for the Predictor, given $\cF$, such that the total estimation error, measured by the expected total Kullback-Leibler (KL) distance 
$$\mathbb{E}_{\tilde{\boldsymbol{Y}}^T}\left[\sum^T_{t=1}\kl(f(\boldx_t),\hat{p}_t=\Phi_t(\boldx_t,\tilde{\boldsymbol{Y}}^{t-1}))\right],$$
over the randomness of the privatization output $\tilde{\boldsymbol{Y}}^T=(\tilde{\boldsymbol{Y}}_1,\ldots,\tilde{\boldsymbol{Y}}_T)=(\cR_1(Y_1),\ldots,\cR_T(Y_T))$, 
is minimized under \emph{arbitrary} choices of $f\in\cF$ and $\boldx^T$ by the Nature. The KL distance $\kl(p_1,p_2)$ is defined as $$\kl(p_1,p_2)=\sum_{y\in\cY}p_1[y]\log\left(\frac{p_1[y]}{p_2[y]}\right)$$ for any distributions $p_1,p_2\in D(\cY)$.
 We define such minimized expected total KL distance as the minimax KL-risk:
 \begin{align*}
 &r^{\kl}_T(\cF) \\
 =&\inf_{\cR^T\in \cL(\epsilon,\delta),\Phi^T}\sup_{f\in\cF,\boldx^T\in\cX^T}\mathbb{E}_{\tilde{\boldsymbol{Y}}^T}\left[\sum^T_{t=1}\kl(f(\boldx_t),\hat{p}_t)\right],
 \end{align*}
where $\cL(\epsilon,\delta)$ is the set of all $(\epsilon,\delta)$-locally differentially private schemes $\cR^T$ satisfying \eqref{eq:ldponline}. 

In addition to minimizing the expected total KL-distance, we are interested in designing a privatization scheme $\cR^T\in\cL(\epsilon,\delta)$ and a prediction scheme $\Phi^T$, such that the corresponding average total variation 
$$\frac{\sum^T_{t=1}|f(\boldx_t)-\hat{p}_t|_{\mathsf{TV}}}{T},$$
under arbitrary choices of $f\in\cF$ and $\boldx^T\in\cX^T$ is upper bounded. Here, the total variation $|p_1-p_2|_{\mathsf{TV}}$ between distributions $p_1,p_2\in\cY$ is defined as  
 $\sum_{y\in\cY}\max\{p_1[y]-p_2[y],0\}$. More specifically, 
for any privatization scheme $\cR^T\in\cL(\epsilon,\delta)$ and prediction scheme $\Phi^T$, let
\begin{align*}   \bar{r}^{\mathsf{TV}}_T(\Phi^T,\cR^T,\cF)=\sup_{f\in\cF,\boldx^T}\mathbb{E}_{\tilde{\mathsf{Y}}^T}\left[\frac{\sum^T_{t=1}|f(\boldx_t)-\hat{p}_t|_{\mathsf{TV}}}{T}\right]
\end{align*}
be the worst-case average total variation associated with $\Phi^T$ and $\cR^T$.

\section{An $\Omega(\sqrt{KT})$ Lower Bound}
\label{sec:lower}
In this section we demonstrate an $\Omega(\frac{1}{\epsilon}\sqrt{KT})$ lower bound by constructing a \emph{hard} hypothesis class $\mathcal{F}$ of size $K$ with $|\mathcal{Y}|\le K$. This will demonstrate that an $\tilde{O}(\sqrt{T\log K})$ type upper bound, such as~\citet{wu2023learning}, is not achievable for \emph{unbounded} label set $\mathcal{Y}$ that may depend on $K$ and $T$. We then provide nearly matching (upto poly-logarithmic factors) upper bounds for \emph{any} finite class $\mathcal{F}$ with adversarially generated features $\x^T$ in Section~\ref{sec:approxdp} and~\ref{sec:puredp}. 

 Before presenting our lower bound, we first demonstrate how the randomized response mechanism, such as~\citet{wu2023learning}, fails to obtain tight KL-risk bounds for unbounded label set. Let $|\mathcal{Y}|=M$, the randomized response mechanism proceeds as follows: for any $y\in \mathcal{Y}$, we set $\tilde{Y}=y$ w.p. $1-\eta$ and w.p. $\eta$ we set $\tilde{Y}$ be uniform  in $\mathcal{Y}\backslash \{y\}$. In order to achieve $(\epsilon,0)$-differential privacy, one would need to set $\eta=\left(\frac{e^{\epsilon}}{M-1}+1\right)^{-1}$. It was demonstrate by~\citet{wu2023learning} that the KL-risk grows as $\tilde{O}(\frac{1}{c_{\eta}}\sqrt{T\log K})$, where $c_{\eta}=1-\frac{M\eta}{M-1}$. This gives $c_{\eta}=\frac{e^{\epsilon}-1}{e^{\epsilon}+M-1}\sim \frac{\epsilon}{M}$ for small $\epsilon$. Therefore, the upper bound grows as $\tilde{O}(\frac{M}{\epsilon}\sqrt{T\log K})$, which will be vacuous for $M\rightarrow \infty$.

We now ready to state our main result of this section:
\begin{theorem}\label{thm:lowerbound}
There exists a finite class $\mathcal{F}\subset \D^{\mathcal{X}}$ of size $K$ with $|\mathcal{Y}|\le K$, such that for any $(\epsilon,0)$-locally private online learning scheme $\cR^T$ and $\Phi^T$ (depending on $\mathcal{F}$), the KL-risk $r_T^{\kl}(\mathcal{F})$ is lower bounded by
 $$\Omega\left(\sqrt{\frac{KT}{\min\{(e^\epsilon-1)^2,1\}e^{\epsilon}}}\right)$$ for all $T\ge \frac{K}{9\min\{(e^\epsilon-1)^2,1\}e^{\epsilon}}$. Moreover, the bound grows as $\Omega(\frac{1}{\epsilon}\sqrt{KT})$ for sufficiently small $\epsilon$.
\end{theorem}

\begin{proof}[Sketch of Proof]
At a high level, for any $K$, our goal is to construct $K$ pairs of distributions $\{p_{k,1},p_{k,2}\}$ for $k\in [K]$, such that: (i) for any $k\in [K]$, $\inf_{\hat{p}}\sup\{\kl(p_{k,1},\hat{p}),\kl(p_{k,2},\hat{p})\}\ge \Omega(a)$ where $a$ is of order $\frac{1}{\epsilon}\sqrt{\frac{K}{T}}$; (ii) $p_{k,1}-p_{k,2}=\frac{a}{N}\boldh_k$ where each $\boldh_k$ corresponds to distinct columns of a Hadamard matrix of dimension $K+1\le N\le 2K$ (excluding the column with all $1$s). Assume such a construction exists (see Appendix~\ref{app:proofthm1}), we then construct the class $\mathcal{F}$ of $2K$ \emph{constant} functions $p_{k,i}$ for $k\in [K]$, $i\in\{1,2\}$. Our key technical part is to show that, for \emph{any} $(\epsilon,0)$-local differential private mechanism (possibly depending on $\mathcal{F}$) there exists $k^*\in [K]$  such that $\kl(\tilde{p}_{k^*,1},\tilde{p}_{k^*,2})\le O(\frac{a^2\epsilon^2}{K})\le \frac{c}{T}$ for some $c<1$, where $\tilde{p}_{k^*,i}$ is the distribution of the private outcomes with input distribution $p_{k^*,i}$. This is achieved by relating the KL-divergence with $\chi^2$ divergence and a careful application of Parseval's identity using the fact that $\boldh_k$s are \emph{orthogonal}. However, this implies, by Pinsker's inequality, that $|\tilde{p}_{k^*,1}^{\otimes T}-\tilde{p}_{k^*,2}^{\otimes T}|_{\mathsf{TV}}<1$ and therefore by Le Cam's two point method no predictor can distinguish $p_{k^*,1},p_{k^*,2}$ from its privatized samples of length $T$. Therefore no algorithm could achieve the KL-risk of order $o(\frac{1}{\epsilon}\sqrt{KT})$ due to property (i) of our construction of $p_{k,i}$s (since any such predictor can be used to distinguish $p_{k^*,1}$, $p_{k^*,2}$). This will then complete the proof. See Appendix~\ref{app:proofthm1} for a  detailed proof.
\end{proof}


Note that our proof is similar in spirit to those employed in~\citet{ullman2018tight} (see also~\citet{gopi2020locally}), e.g., the crucial application of the Parseval's identity. However, a distinguishing feature of our construction is that our label set $\mathcal{Y}$ is of size $K$, while~\citet{ullman2018tight} requires a $2^K$ support size. Moreover, our proof is constructive and employs methodology that is  suitable for bounding KL-risk instead of the total variation, which is of  independent interest.

\section{Approximate-DP via WMA}
\label{sec:approxdp}
In this section, we present an online learning scheme (Algorithm \ref{alg:WM}) that is $(\epsilon,\delta)$-locally differentially private and has minimax KL-risk $r^{\kl}_T(\cF)=\tilde{O}\left(\frac{1}{\epsilon}\sqrt{TK\log\frac{1}{\delta}}\right)$, where $\tilde{O}$ omits the $poly(\log (KT))$ factors. Our scheme makes use of the Laplacian mechanism for local differential privacy 
and the weighted majority algorithm (WMA) for online learning. 

\subsection{The weighted majority algorithm}

Before presenting our scheme, we first briefly review the well-known weighted majority algorithm
for the following online game: Suppose there are $K$ experts $[K]$, accessible by the Predictor. At each time $t\in[T]$, the Predictor picks an expert $i_t$  $i_t\in[K]$ and observes a loss vector $\boldv_t=(v_{t,1},\ldots,v_{t,K})\in[0,1]^K$, which represents the loss of choosing each of the experts. Then, the Predictor experiences loss $v_{t,i_t}$. Given $K$ and $T$, the weighted majority algorithm is as follows
\begin{enumerate}
    \item[]\textbf{Initialization:} Let $\boldw^1=(w^1_1,\ldots,w^1_K)=(1,\ldots,1)$ and $\eta=\sqrt{\frac{2\log K}{T}}$.
    \item[] \textbf{For} time $t=1,\ldots,T$
    \item[] \textbf{Step 1:} For $j\in[K]$, the Predictor select $i_t=j$ with probability
    $\Tilde{w}^t_j=\frac{w^t_j}{\sum^K_{i=1}w^t_i}$
    \item[] \textbf{Step 2:} The Predictor observes the loss vector $\boldv_t$ and experiences loss $v_{t,i_t}$
    \item[] \textbf{Step 3: Update} $w^{t+1}_j=w^t_j*e^{-\eta v_{t,j}}$ for $j\in[K]$. 
\end{enumerate}
The following lemma gives an upper bound on the regret of WMA, defined as the expected loss minus the minimum total loss of a single expert
\begin{lemma}[\citet{shalev2014understanding}] \label{lem:wmalemma} When $\boldv_t\in[0,1]^K$ for $t\in[T]$, the regret of WMA satisfies
 \begin{align*}
     \sum^T_{t=1}\sum^K_{j=1}\Tilde{w}^t_jv_{t,j}-\min_{i\in[K]}\sum^T_{t=1}v_{t,i}\le \sqrt{2T\log K}.
 \end{align*}
\end{lemma}
\subsection{Our scheme}
Our key idea of leveraging the WMA algorithm in the context of minimizing KL-risk is to appropriately define the \emph{loss} vector $\mathbf{v}_t$s.
At a high level, we will choose $\mathbf{v}_t[j]$ to be the \emph{log-likelihoods} $-\log f_j(\x_t)[Y_t]$ for $j\in [K]$ and add Laplacian noise to the loss vectors. However, instead of selecting an expert (or a distribution from a set of candidates as in the online distribution learning setting) at a time, we estimate the distribution using a \emph{weighted average}.

We first define random functions for "clipping" the log-likelihood, which is crucial for applying the Laplacian mechanism. The "clipping" technique appeared in \citet{gopi2020locally}. However, here we use slightly different "clipping functions" for the purpose of bounding the KL divergence. 

\paragraph{Clipping of distributions.}  Let $\x^T$ be any realization of features.  For each $t\in[T]$, let 
    \begin{align}\label{eq:ntprime}
        N'_t=\sum_{y\in\cY}\lceil M\max_{j\in[K]}f_j(\boldx_t)[y] \rceil.
    \end{align} 
    Let 
    $\{\cS_{y,t}\}_{y\in\cY}$, $\cS_{y,t}\subset [N'_t]$, be a partition of $[N'_t]$  such that $|\cS_{y,t}|=\lceil M\max_{j\in[K]}f_j(\boldx_t)[y] \rceil$. Define the \emph{random} mapping $h'_t:\cY\rightarrow [N'_t]$ such that $h'_t(y)$ is uniformly distributed over $\cS_{y,t}$, i.e., 
    \begin{align}\label{eq:hprime}
    \mathrm{Pr}(h'_t(y)=y'|y)=\begin{cases}
        \frac{1}{\lceil M\max_{j\in[K]}f_j(\boldx_t)[y]\rceil}, & \mbox{ if $y'\in\cS_{y,t}$}\\
        0,&\mbox{ else.}
    \end{cases}
    \end{align}
Let $h'_t\circ f_j(\boldx_t)$, $j\in[K]$, $t\in[T]$ be the  distribution of $h'_t(y)$ when  $y\in\cY$ is sampled from distribution $f_j(\boldx_t)$. 

Let also $h_t(y):\cY\rightarrow[N'_t]$ be $h'_t(y)$ with probability $1-\frac{1}{T}$ and be $U(y)$ with probability $\frac{1}{T}$,  where $U(y)$ is uniformly distributed over $[N'_t]$, i.e., 
\begin{align}\label{eq:hfunction}
    &\mathrm{Pr}(h_t(y)=y'|y)\nonumber\\
    =&\begin{cases}
        \frac{1-\frac{1}{T}}{\lceil M\max_{j\in[K]}f_j(\boldx_t)[y]\rceil}+\frac{1}{TN'_t}, & \mbox{ if $y'\in\cS_{y,t}$}\\
        \frac{1}{TN'_t},&\mbox{ else.}
    \end{cases}        
    \end{align}
Similarly, define $h_t\circ f_j(\boldx_t)$, $j\in[K]$, $t\in[T]$
 as the distribution of $h_t(y)$, $y\in\cY$ when $y$ is drawn according to distribution $f_j(\boldx_t)$, i.e.,
\begin{align}\label{eq:hcircf}
    h_t\circ f_j(\boldx_t)[y']=\sum_{y\in\cY}\mathrm{Pr}(h_t(y)=y'|y)f_j(\boldx_t)[y].
\end{align} 
By definition of $N'_t$ \eqref{eq:ntprime}, we have $M\le N'_t\le KM$. Therefore, 
 \begin{align}\label{eq:clipping}
     \frac{1}{TKM}\le \big(h_t\circ f_j(\boldx_t)\big)[y']\le \frac{1}{M}
 \end{align}
for any $j\in[K]$ and $y'\in [N'_t]$.  The fact that $h_t\circ f_j(\boldx_t)$ is upper and lower bounded \eqref{eq:clipping} implies that the loglikelihood $-\log (h_t\circ f_j(\boldx_t)\big)[y'])$ has sensitivity\footnote{For a real-valued function $s:\cZ\rightarrow\mathbbm{R}$, its sensitivity $\Delta_1(s)$ is defined as $\max_{z\in \cZ}s(z)-\min_{z\in\cZ}s(z)$.} $\log (KT)$ and thus can be made $(\epsilon,0)$ differentially private by adding Laplacian noise with scale\footnote{A Laplacian random variable with scale $b$ has density $\frac{1}{2b}e^{\frac{-|x|}{b}}$ for $x\in\mathbbm{R}$ .} $\frac{\log (KT)}{\epsilon}$ \citep{dwork2014algorithmic}.

\paragraph{The privatization scheme.}In order to preserve the privacy for the loss vector defined in WMA, we need to add Laplacian noise to a vector instead of a scalar value. The following lemma shows that, by using  advanced composition, one only needs to add Laplacian noises with the scale which is \emph{square-root} of the dimension.
\begin{lemma}[\citet{steinke2022composition}]\label{lem:advancedcomposition1}
    Let $A_1,\ldots,A_K:\cY\rightarrow \mathbbm{R}$ be $K$ random algorithms that are $(\epsilon',0)$ diferentially private. Then the composition $A(y)=(A_1(y),\ldots,A_K(y)):\cY\rightarrow \mathbbm{R}^K$ where $A_1,\ldots,A_K$ run independently, is $\left(\frac{K(\epsilon')^2}{2}+\sqrt{2\ln(\frac{1}{\delta})K(\epsilon')^2},\delta\right)$ diferentially private for any $\delta>0$.
\end{lemma}
We are now ready to present our privatization scheme. Let $\x_t$ and (random) $Y_t$ be the feature and label at time $t$, and $\bar{p}_{t,j}=h_t\circ f_j(\x_t)$ for $j\in [K]$.  
Let $\gamma>0$ be a number to be determined later, we define the privatized vector $\mathbf{\tilde{Y}}_t=(\tilde{Y}_{t,1},\cdots,\tilde{Y}_{t,K})$ as
\begin{align}\label{eq:tildey}
\Tilde{Y}_{t,j}\nonumber=-c(\gamma)(\log \big(\bar{p}_{t,j}[h_t(Y_t)]\big)&+Lap^t_j\\
&+\log M-c'(\gamma)),
\end{align}
where $Lap^t_j$, $t\in[T]$, $j\in[K]$ are independent random variables having Laplacian distribution with scale $\frac{(2\sqrt{2K\ln \frac{1}{\delta}}+\sqrt{K\epsilon})\log (KT) } {\epsilon}$. Moreover,  $c(\gamma)=\frac{1}{\log (KT)+2 c'(\gamma)}$ and $c'(\gamma)$ is
\begin{align}\label{eq:cgamma}
    \frac{(2\sqrt{2K\ln \frac{1}{\delta}}+\sqrt{K\epsilon})\log (KT)(\gamma+\log K+\log T) }{\epsilon}.
\end{align}
Note that, the loss vector $\Tilde{\boldsymbol{Y}}_{t}=(\Tilde{Y}_{t,1},\ldots,\Tilde{Y}_{t,K})$, $t\in[T]$ is a privatized version of the loglikelihood vector 
$(\bar{p}_{t,1}[h_t(Y_t)],\ldots,\bar{p}_{t,T}[h_t(Y_t)])$ and thus a privatization of $Y_t$. 
The following lemma shows that the privatized data $\Tilde{\boldsymbol{Y}_t}$ satisfies the privacy constraints.  
\begin{lemma}\label{lem:ldpfortildey}
 The privatization $\cR_t(Y_t)=\Tilde{\boldsymbol{Y}}_t$, $t\in[T]$ in Algorithm \ref{alg:WM} is $(\epsilon,\delta)$-locallly differentially private.
\end{lemma}
\begin{proof}
Note that from \eqref{eq:clipping}, the sensitivity of $\log \big(\bar{p}_{t,j}[h_t(Y_t)]\big)$ is 
$\log (KT)$ for $j\in[K]$, $t\in[T]$. Hence from \eqref{eq:tildey}, $\Tilde{Y}_{t,j}$ is $(\epsilon'=\frac{\epsilon}{2\sqrt{K\ln (\frac{1}{\delta})}+\sqrt{K\epsilon}},0)$-locally differentially private. Invoking Lemma \ref{lem:advancedcomposition1}, $\Tilde{\boldsymbol{Y}}_t=(\Tilde{Y}_{t,1},\ldots,\Tilde{Y}_{t,K})$ is $(\frac{K(\epsilon')^2}{2}+\sqrt{2\ln(\frac{1}{\delta})K(\epsilon')^2},\delta)$ differentially private, and thus is $(\epsilon,\delta)$-diferentially private since $$\frac{K(\epsilon')^2}{2}+\sqrt{2\ln(\frac{1}{\delta})K(\epsilon')^2}\le \epsilon.$$  
Hence, $\cR^T$ is $(\epsilon,\delta)$-locally diferentially private.
\end{proof}

\paragraph{Distribution learning algorithm.}
We now introduce our private learning algorithm, as presented in Algorithm~\ref{alg:WM}, using our privatization scheme discussed above. Note that the prediction $\hat{p}_t$ at time $t$ may be different from any of the experts $h_t\circ f_j(\boldx_t)$ and the candidate distributions $f_j(\boldx_t)$, $j\in[K]$.
\begin{algorithm}
  \caption{Locally Privatized WMA\iffalse$(\boldsymbol{a},\boldsymbol{b},\boldsymbol{c})$\fi} 
  \label{alg:WM}
  \begin{algorithmic}[1]
    \STATE \textbf{Input:} The hypothesis set $\mathcal{F}$, the time horizon $T$, probability parameter $\gamma$
    \STATE \textbf{Initialize:} Let $K=|\mathcal{F}|$, $\boldsymbol{w}^1=(w^1_1,\ldots,w^1_K)=(1,\ldots,1)$, and $\eta = \sqrt{\frac{2K\log K}{T}}$;
    \FOR {$t=1,\ldots,T$}
    \STATE Receive feature $\mathbf{x}_t$; 
    \STATE Set $\bar{p}_t=\sum_{j=1}^K\tilde{w}_j^t\bar{p}_{t,j}$ where $\bar{p}_{t,j}=h_t\circ f_j(\mathbf{x}_t)$ is defined in \eqref{eq:hcircf} and $\tilde{w}^t_j=\frac{w^t_j}{\sum^K_{j=1}w^t_j}$;
    \STATE Make prediction $\hat{p}_t[y]=\sum_{y'\in \mathcal{S}_{y,t}}\bar{p}_t[y']$ for all $y\in\mathcal{Y}$, where $\mathcal{S}_{y,t}$ is defined in \eqref{eq:hprime};
    \STATE Receive privatized data $\mathbf{\tilde{Y}}_t=(\tilde{Y}_{t,1},\cdots,\tilde{Y}_{t,K})$ as defined in \eqref{eq:tildey};  
    \STATE \textbf{Update} $w^{t+1}_j=w^t_j*e^{-\eta\tilde{Y}_{t,j}}$ for $j\in [K]$;
    \ENDFOR
  \end{algorithmic}
\end{algorithm}

We now establish an $\tilde{O}(\sqrt{TK})$ upper bound on the KL-risk of Algorithm \ref{alg:WM}. The following lemma directly follows from Lemma \ref{lem:wmalemma}.
\begin{lemma}\label{lem:regretbound}
In Algorithm \ref{alg:WM}, 
if $\Tilde{Y}_{t,j}\in [0,1]$ for $j\in[K]$ and $T>\log K$, then
$$\sum^T_{t=1}\sum^K_{j=1}\Tilde{w}^t_j\Tilde{Y}_{t,j}-\min_{i\in[K]}\Big(\sum^T_{t=1}\Tilde{Y}_{t,i} \Big)\le \sqrt{2T\log K}.$$
\end{lemma}
We now introduce the following key lemma, which establishes an upper bound for the KL-risk condition on the event that the Laplacian noises are bounded.
\begin{lemma}\label{lem:wmbound}
For any $f\in \mathcal{F}$ and $\x^T\in \mathcal{X}^T$, with probability at least $1-e^{-\gamma}$ on the randomness of $Lap_{j}^t$s, the output $\hat{p}_t$ of Algorithm \ref{alg:WM} satisfies for any $\gamma>0$
\begin{align}\label{eq:wmbound}
    \mathbb{E}_{Y'^T}&\left[\mathbb{}\sum^T_{t=1}\kl(f(\boldx_t),\hat{p}_t)\right]\nonumber\\
    &\le \frac{1}{c(\gamma)}\sqrt{2T\log K}+3\log(KT)\nonumber\\
    &+\sum^T_{t=1}\sum^K_{j=1}\mathbb{E}_{Y'^T}[\Tilde{w}^t_j]Lap^t_j-\sum^T_{t=1}Lap^t_{j^*},
\end{align}
where $c(\gamma)$ is given in \eqref{eq:cgamma} and $Y'_t\sim h_t\circ f(\x_t)$.
\end{lemma}
\begin{proof}[Sketch of Proof]
We sketch only the high level idea here and refer to Appendix~\ref{ap:prooflem5} for a detailed proof. Note that, in order to apply Lemma~\ref{lem:regretbound} one must ensure that the loss vector $\tilde{\mathbf{Y}}_t$ to be within $[0,1]$ for each coordinates. We show that this happens w.p. $\ge 1-e^{-\gamma}$ by our definition of $\tilde{Y}_{t,j}$s in~(\ref{eq:tildey}) and the concentration property of Laplacian distributions. Now, conditioning on such an event, we show by Lemma~\ref{lem:regretbound} and definition of $\mathbf{\tilde{Y}}_t$ that for any $j^*\in [K]$ we have $\sum_{t=1}^T-\log(\bar{p}_t[Y'_t])+\log(\bar{p}_{j^*,t}[Y'_t])\le \frac{1}{c(\gamma)}\sqrt{2T\log K}+\beta$, where $\beta$ depends only on the Laplacians $Lap_j^t$ and $Y'_t=h_t(Y_t)$. Assuming $Y_t\sim f_{j^*}(\x_t)$ and taking expectation on the randomness of $Y_t'$s, we have $\mathbb{E}_{Y'^T}\left[\sum_{t=1}^T\kl(\bar{p}_{j^*,t},\bar{p})\right]\le \frac{1}{c(\gamma)}\sqrt{2T\log K}+\beta$. Here, we have used the key observation that $\mathbb{E}_{y\sim p}\log(p[y]/q[y])=\kl(p,q)$ and the law of total probability. The lemma then follows by a careful analysis that relates $\kl(f_{j^*}(\x_t),\hat{p}_t)$ and $\kl(\bar{p}_{j^*,t},\bar{p})$ by leveraging the crucial properties for our "clipping functions". See Appendix~\ref{ap:prooflem5} for details.
\end{proof}
We now ready to state our main result of this section:
\begin{theorem}\label{thm:klriskepsdeltadp} For any class $\mathcal{F}$ of size $K$ there exist an $(\epsilon,\delta)$-local differential private mechanism that achieves the KL-risk $r^{\kl}_T(\cF)$ upper bounded by $O\Bigg(\sqrt{2T\log K}\Big(\log (KT)+\frac{(\sqrt{K\ln\frac{1}{\delta}}+\sqrt{K\epsilon})\log^2(KT)}{\epsilon}\Big)\Bigg)$
\end{theorem}
\begin{proof}
Let $\cE$ be the event that for all $j,t$, $|Lap_{j}^t|\le c'(\gamma)$, which happens w.p. $\ge 1-e^{-\gamma}$ and implies that Lemma~\ref{lem:wmbound} holds, see Appendix~\ref{ap:prooflem5}. We have
\begin{align*}
    r_T^{\kl}(\mathcal{F})&=\mathbb{E}_{\tilde{\boldsymbol{Y}}^T}\Bigg[\sum^T_{t=1}\kl(f_{j^*}(\boldx_t),\hat{p}_t)\Bigg]\\ &=\mathrm{Pr}(\cE)\mathbb{E}_{\tilde{\boldsymbol{Y}}^T}\Bigg[\sum^T_{t=1}\kl(f_{j^*}(\boldx_t),\hat{p}_t)|\cE\Bigg]\\
&\quad+\mathrm{Pr}(\cE^c)\mathbb{E}_{\tilde{\boldsymbol{Y}}^T}\Bigg[\sum^T_{t=1}\kl(f_{j^*}(\boldx_t),\hat{p}_t)|\cE^c\Bigg].
\end{align*}
Denote by $(A)$ and $(B)$ the two terms in the above expression that correspond to $\cE$ and $\cE^c$, respectively. By Lemma~\ref{lem:wmbound}, we know $(A)$ can be upper bounded by
\begin{align}
\label{eq:laplace}
    \mathbb{E}_{\tilde{\boldsymbol{Y}}^T}\Bigg[\frac{1}{c(\gamma)}&\sqrt{2T\log K}+3\log(KT)\nonumber\\
    &+\sum^T_{t=1}\sum^K_{j=1}\Tilde{w}^t_jLap^t_j-\sum^T_{t=1}Lap^t_{j^*}|\cE^c\Bigg].
\end{align}
Note that, conditioning on $\cE$ the Laplacians $Lap^{t}_j$s are still $i.i.d.$ distributed and the $\tilde{w}_j^t$s are summing to $1$, so that ~\eqref{eq:laplace} vanishes when taking (conditional) expectation. Therefore, $(A)$ is upper bounded by $\frac{1}{c(\gamma)}\sqrt{2T\log K}+3\log(KT)$. 

To bound the term $(B)$, we have by~\eqref{eq:klhp} and~\eqref{eq:klhprimebarpt} (in Appendix~\ref{ap:prooflem5}) that $\kl(f_j^*(\x_t),\hat{p}_t)$ is upper bounded by
$\kl(\bar{p}_{t,j^*},\bar{p}_t)+O\left(\frac{\log(KT)}{T}\right).$
Noticing that $\frac{\bar{p}_{t,j^*}[y']}{\bar{p}_t[y']}\le KT$ for all $y'\in [N'_t]$ (see~\eqref{eq:clipping}), we conclude $\kl(\bar{p}_{t,j^*},\bar{p}_t)\le \log(KT)$. Therefore, by summing over $t\in [T]$, the term $(B)$ is upper bounded by
$e^{-\gamma}(T\log(KT)+O(\log(KT)))$, since $\mathrm{Pr}[\cE^c]\le e^{-\gamma}$. Putting everything together,  the KL-risk $r_T^{\kl}(\mathcal{F})$ is upper bounded by
$$\frac{1}{c(\gamma)}\sqrt{2T\log K}+3\log(KT)+O(e^{-\gamma}T\log(KT)).$$
The theorem now follows by setting $\gamma=\log T$ and by the definition of $c(\gamma)$ in~\eqref{eq:cgamma}.
\end{proof}
Observe that the KL-risk bound in Theorem~\ref{thm:klriskepsdeltadp} is \emph{independent} of the label set size $M$ and grows as $$O\left(\frac{1}{\epsilon}\sqrt{TK\log^5(KT)\log\frac{1}{\delta}}\right)$$ for sufficiently small $\epsilon$ and $\delta$.

\section{Pure-DP via Modified EXP3}
\label{sec:puredp}
While Algorithm \ref{alg:WM} attains \((\epsilon,\delta)\)-local differential privacy, it is worthwhile to investigate whether it is possible to attain the stronger, pure \((\epsilon,0)\)-differential privacy while still achieving comparable KL-risk bounds. We demonstrate in this section that the answer is affirmative, and it is possible to achieve the same \(\tilde{O}\left(\frac{1}{\epsilon}\sqrt{TK}\right)\) KL-risk bound. By Theorem~\ref{thm:lowerbound}, we known that this is essentially tight for pure-DP constrains.

Note that the reason why Theorem~\ref{thm:klriskepsdeltadp} has a dependency on $\delta$ is due to the \emph{advanced composition} lemma (Lemma~\ref{lem:advancedcomposition1}) that allows us to select the scale of the Laplacians with order of $\sqrt{K}$, which is essential for achieving a $\sqrt{KT}$ type bound. To resolve this issue, we introduce a new privatization mechanism in this section by selecting a single \emph{random} component in the log-likelihood vectors and reveal only the privatized version of such a component, as in  Algorithm~\ref{alg:bandit}. Note this reduces significantly the scale of Laplacian from $\sqrt{K}$ to $O(1)$. However, since we only return a single component of the loss vector, the WMA algorithm will not apply. To resolve this issue, we employ an analysis similar to the EXP3 (Exponential-weight for Exploration and Exploitation) algorithm designed for bandit learning via an unbiased estimation of the loss vector. Note that the main difference with the standard EXP3 algorithm is that we \emph{do not} reveal the loss of the expert selected by predictor and instead we real the loss for a \emph{randomly} selected expert. This is crucial for making our privatization mechanism \emph{independent} of prior history, i.e., using only private coins.


We now describe our privatization scheme. Let $\x_t$ and $Y_t$ be the feature and label at time $t$, and $\bar{p}_{t,j}=h_t\circ f_j(\x_t)$ as in~\eqref{eq:hcircf}. 
Let $\gamma>0$ to be determined later, we define the vector $\mathbf{Z}_{t,j}=(Z_{t,j,1},\cdots,Z_{t,j,K})$ such that
\begin{align}\label{eq:tildeypdp}
Z_{t,j,i}=\begin{cases}
-c_{pdp}(\gamma)(\log \big(\bar{p}_{t,j}[h_t(Y_t)]\big)+\\
Lap^t_j+\log M-c'_{pdp}(\gamma)), &\mbox{if $i=j$}\\
0,&\mbox{else}.
\end{cases}
\end{align}
where $Lap^t_j$ are independent random variables having Laplacian distribution with scale $\frac{\log (KT)} {\epsilon}$ and $h_t$ is the random function as in~\eqref{eq:hfunction}. Moreover, the numbers $c_{pdp}(\gamma)=\frac{1}{\log (KT)+2 c_{pdp}'(\gamma)}$ and
\begin{align}\label{eq:cpdpgamma}
    c_{pdp}'(\gamma)=\frac{\log (KT)(\gamma+\log K+\log T) }{\epsilon}.
\end{align}
Let $\Tilde{\boldsymbol{Y}}_{t}$ be a \emph{random} vector that equals $\mathbf{Z}_{t,j}$ with probability $\frac{1}{K}$ for each $j\in[K]$. We define $\tilde{\mathbf{Y}}_t$ as the privatized version of $Y_t$. The following lemma demonstrates that our scheme is indeed $(\epsilon,0)$-differential private.
\begin{lemma}\label{lem:pldpfortildey}
 The privatization $\cR_t(Y_t)=\Tilde{\boldsymbol{Y}}_t$, $t\in[T]$ in Algorithm \ref{alg:bandit} is $(\epsilon,0)$-locallly differentially private.
\end{lemma}
\begin{proof}
    Since for each $t$, only a single entry of $\Tilde{\boldsymbol{Y}}_t$ is non-zero, and $\log(\bar{p}_{t,j}[h_t(Y_t)])$ has sensitivaity $\log(KT)$, the Laplacian mechanism with scale $\frac{\log(KT)}{\epsilon}$ provides $(\epsilon,0)$ differentially privacy~\citep{dwork2014algorithmic}.
\end{proof}
We present our learning algorithm in Algorithm~\ref{alg:bandit}.
\begin{algorithm}
  \caption{Pure Locally Differentially Private\iffalse$(\boldsymbol{a},\boldsymbol{b},\boldsymbol{c})$\fi} 
  \label{alg:bandit}
  \begin{algorithmic}[1]
    \STATE \textbf{Input:} The hypothesis set $\cF$, the time horizon $T$, probability parameter $\gamma$
    \STATE\textbf{Initialize:} Let $K=|\cF|$, $\boldw^1=(w^1_1,\ldots,w^1_K)=(1,\ldots,1)$, and $\eta = \sqrt{\frac{2K\log K}{T}}$;
    \FOR{$t=1,\ldots,T$}
    \STATE Receive feature $\boldx_t$; 
    \STATE Set $\bar{p}_t=\sum_{j=1}^K\tilde{w}_j^t\Bar{p}_{t,j}$, where $\Tilde{w}^t_j=\frac{w^t_j}{\sum^K_{j=1}w^t_j}$. 
    \STATE Make prediction $\hat{p}_t[y]=\sum_{y'\in \cS_{y,t}}\bar{p}_t[y']$ for all $y\in\cY$, where $\cS_{y,t}$ is defined in \eqref{eq:hprime};
    \STATE Receive privatized data $\Tilde{\mathbf{Y}}_{t}$, where $\Tilde{\mathbf{Y}}_{t}=\mathbf{Z}_{t,J_t}$ as in \eqref{eq:tildeypdp} and $J_t$ is \emph{uniform} over $[K]$;
    \STATE \textbf{Update} $w^{t+1}_j=w^{t}_j*e^{-\eta \Tilde{Y}_{t,j}}$ for $j\in[K]$;
    \ENDFOR
  \end{algorithmic}
\end{algorithm}

We first prove the following key lemma that extends lemma~\ref{lem:regretbound} to the \emph{expected} loss vectors.
\begin{lemma}\label{lem:regretboundpdp}
In Algorithm \ref{alg:bandit}, 
if $\mathbf{Z}_{t,j}\in [0,1]^K$ for all $t\in [T]$, $j\in [K]$ and $T>\log K$, then
$$\max_{i\in [K]}\mathbb{E}\left[\sum^T_{t=1}\left(\sum^K_{j=1}\Tilde{w}^t_jZ_{t,j,j}-Z_{t,i,i}\right)\right]\le \sqrt{2TK\log K}$$
where the expectation is over the randomness of $J_t$s.
\end{lemma}
\begin{proof}
Let $\tilde{\mathbf{Y}}_t=(\tilde{Y}_{t,1},\cdots,\tilde{Y}_{t,K})$ be as in Algorithm~\ref{alg:bandit}. We use a more general result of Lemma \ref{lem:wmalemma}, which is implied in the proof of \citet[Thm 21.11]{shalev2014understanding}, that is,
\begin{align*}
&\sum^T_{t=1}\sum^K_{j=1}\Tilde{w}^t_j\Tilde{Y}_{t,j}-\min_{i\in[K]}\Big(\sum^T_{t=1}\Tilde{Y}_{t,i} \Big)\nonumber\\
\le &\frac{\log K}{\eta}+\frac{\eta}{2}\sum^T_{t=1}\sum^K_{j=1}\Tilde{w}^t_j\Tilde{Y}^2_{t,j}
\end{align*}
Taking expectations over $J_t$s, we have
\begin{align}\label{eq:wmauppereta}
&\max_{i\in [K]}\mathbb{E}\left[\sum^T_{t=1}\left(\sum^K_{j=1}\Tilde{w}^t_j\tilde{Y}_{t,j}-\tilde{Y}_{t,i}\right)\right]\nonumber\\
\le &\frac{\log K}{\eta}+\mathbb{E}\left[\frac{\eta}{2}\sum^T_{t=1}\sum^K_{j=1}\Tilde{w}^t_j\tilde{Y}_{t,j}^2\right]\nonumber\\
\overset{(a)}{\le}&\frac{\log K}{\eta}+\frac{\eta}{2}\sum^T_{t=1}\frac{1}{K}=\frac{\log K}{\eta}+\frac{\eta}{2}\frac{T}{K}
\end{align}
where $(a)$ follows by $\mathbb{E}_{J_t}[\tilde{Y}_{t,j}^2]=\frac{Z_{t,j,j}^2}{K}\le \frac{1}{K}$ and that $\tilde{w}_j^t$s are independent of $J_t$.
Taking $\eta=\sqrt{\frac{2K\log K}{T}}$ and noticing that $\mathbb{E}_{J_t}[\tilde{Y}_{t,j}]=\frac{Z_{t,j,j}}{K}$, the lemma follows.
\end{proof}
The following lemma is analogous to Lemma~\ref{lem:wmbound}, and we defer the proof to Appendix~\ref{app:prooflem8}.
\begin{lemma}\label{lem:wmboundpdp}
For any $f\in \mathcal{F}$ and $\x^T\in \mathcal{X}^T$, with probability at least $1-e^{-\gamma}$ on the randomness of $Lap_{j}^{t}$s, the output $\hat{p}_t$ of Algorithm \ref{alg:bandit} satisfies
\begin{align}\label{eq:wmboundpdp}
    &\mathbb{E}_{Y'^T,J^T}\left[\sum^T_{t=1}\kl(f(\boldx_t),\hat{p}_t)\right]\nonumber\\
    \le &\frac{1}{c_{pdp}(\gamma)}\sqrt{2TK\log K}+3\log(KT)\nonumber\\
    &+\sum^T_{t=1}\sum^K_{j=1}\mathbb{E}_{Y'^T,J^T}[\Tilde{w}^t_j]Lap^t_j-\sum^T_{t=1}Lap^t_{j^*},
\end{align}
where $c_{pdp}(\gamma)$ is given in \eqref{eq:cpdpgamma}, $Y'_t\sim h_t\circ f(\x_t)$ and $J_t$s are the random indices as in Algorithm~\ref{alg:bandit}. 
\end{lemma}

We are now ready to state our main result of this section. The proof essentially follows the same steps as the proof of Theorem~\ref{thm:klriskepsdeltadp} and is therefore omitted due to space constraints.
\begin{theorem}\label{thm:klriskinteractive}
For any class $\mathcal{F}$ of size $K$, there exists an $(\epsilon,0)$-local differential private mechanism that achieves the following KL-risk $r_T^{\kl}$ upper bound
    $$O\Bigg(\sqrt{2TK\log K}\Big(\log (KT)+\frac{\log^2(KT)}{\epsilon}\Big)\Bigg).$$
\end{theorem}

\section{Bounding Averaged TV-risk}
In this section, we show how the KL-risk upper bound for our  $(\epsilon,0)$-local differential private algorithm (Theorem \ref{thm:klriskinteractive}) can be applied to obtain bounds for the averaged TV-risk as introduced in Section~\ref{sec:problemform}.

The following result follows directly from Pinsker's inequality~\citep{pw22}.

\begin{theorem}
\label{thm:tvrisk}
    For any class $\mathcal{F}$ of size $K$, there exists a $(\epsilon,0)$-local differential private mechanism that achieves the averaged TV-risk $\bar{r}_T^{\mathsf{TV}}$ upper bounded by
    $$\tilde{O}\left(\sqrt{\frac{1}{\epsilon}\sqrt{\frac{K}{T}}}\right),$$
    for sufficiently small $\epsilon$.
\end{theorem}
\begin{proof}
    Let $\hat{p}^T$ be the predictors that achieve the KL-risk bound of Theorem~\ref{thm:klriskinteractive}. For any $\x^T$ and $f\in \mathcal{F}$, we have by Pinsker's inequality~\citep{pw22} that $|f(\x_t)-\hat{p}_t|_{\mathsf{TV}}\le \sqrt{\kl(f(\x_t),\hat{p}_t)/2}$ for all $t\in [T]$. Therefore, we have
    \begin{align*}
        \frac{1}{T}\sum_{t=1}^T|f(\x_t)-\hat{p}_t|_{\mathsf{TV}}&\le \frac{1}{T}\sum_{t=1}^T\sqrt{\frac{1}{2}\kl(f(\x_t),\hat{p}_t)}\\
        &\le \frac{1}{T}\sqrt{\frac{T}{2}\sum_{t=1}^T\kl(f(\x_t),\hat{p}_t)}\\
        &\le \tilde{O}\left(\sqrt{\frac{1}{\epsilon}\sqrt{\frac{K}{T}}}\right),
    \end{align*}
    where the second inequality follows by Cauchy-Schwartz inequality and the last inequality follows by Theorem~\ref{thm:klriskinteractive}.
\end{proof}

Finally, we find $T_\alpha$ that makes the TV-risk bounded by some $\alpha$. 
Let $\mathcal{Q}=\{p_1,\cdots,p_K\}$ be $K$ distributions. We construct a class $\mathcal{F}$ of $K$ constant functions with values in $\mathcal{Q}$. Denote by $\hat{p}^T$ the predictors that achieves the upper bound of Theorem~\ref{thm:tvrisk} and $\hat{p}=\frac{1}{T}\sum_{t=1}^T\hat{p}_t$. We have by convexity of TV-distance and Jensen's inequality that for any underlying truth $p^*\in \mathcal{Q}$, $|p^*-\hat{p}|_{\mathsf{TV}}\le \tilde{O}\left(\sqrt{\frac{1}{\epsilon}\sqrt{\frac{K}{T}}}\right)$. Taking $T_\alpha=\tilde{O}\left(\frac{K}{\epsilon^2\alpha^4}\right)$ one can make the TV-risk upper bounded by $\alpha$. This can be boosted to high probability via a median trick, and therefore recovers the nearly tight upper bound in~\citet[Thm 3]{gopi2020locally} with \emph{non-interactive} mechanisms.
\begin{remark}
    We conjecture that the averaged TV-risk in Theorem~\ref{thm:tvrisk} may not be tight. We leave it as an open problem to determine if such an upper bound can be improved to $\tilde{O}(\frac{1}{\epsilon}\sqrt{\frac{K}{T}})$. 
    Note that, it was demonstrated by~\citet{gopi2020locally} via suitable comparison schemes that the sample complexity of the \emph{batch} setting is upper bounded by $\tilde{O}\left(\frac{K}{\epsilon^2\alpha^2}\right)$ using interactive private mechanisms. It is therefore interesting to investigate if such comparison based arguments can be extended to our online case.
\end{remark}
\newpage
\bibliographystyle{apalike}
\bibliography{reference}

\onecolumn
\appendix
\thispagestyle{empty}

\section{Proof of Theorem~\ref{thm:lowerbound}}
\label{app:proofthm1}
Let $N=2^n$ for some positive integer $n$ such that $\frac{N}{2}\le K\le N-1$. Let $$H^N=\begin{bmatrix}
1 & 1\\
1 & -1
\end{bmatrix}^{\otimes n}$$
be a Hadamard matrix and $H^N_{i}$ be the $i$th column of $H^N$. Let $\cY=[N]$ and consider the following $2K$ distributions:
\begin{align}\label{eq:distributionp}
p_{i,1}[y]=&\begin{cases}
    0,&\mbox{ if $H^N_{i+1,y}=1$,}\\
    \frac{2}{N},&\mbox{ if $H^N_{i+1,y}=-1$}
\end{cases}\nonumber\\
p_{i,2}[y]=& p_{i,1}[y]+\frac{H^N_{i+1,y}a}{N}, \ \ \ \ \  a=\sqrt{\frac{K}{9T\min\{(e^\epsilon-1)^2,1\}e^{\epsilon}}}
\end{align}
for $i\in[K]$, $y\in[N]$ 
Consider learning a hypothesis set $\cF=\{f_{1,1},f_{1,2},f_{2,1},\ldots,f_{K,1},f_{K,2}\}$, where $f_{i,\ell}(\boldx)=p_{i,\ell}$ for all $\boldx\in\cX$, $i\in[K]$, and $\ell\in[2]$.
 The samples $Y^T$ are i.i.d. random variables generated according to some distribution $p\in\{p_{i,\ell}\}_{i\in[K],\ell\in[2]}$. The goal is to give an estimate of $\hat{p}$ based on locally privatized data $\tilde{\boldsymbol{Y}}^T$, where $\tilde{\boldsymbol{Y}}_t=\cR_t(Y_t)$, such that the KL distance $\kl(p,\hat{p})$ is minimized. Note that one can use a private online learning predictor $\Phi^T$ to make the estimate $\hat{p}$ as follows: set $\hat{p}=\frac{1}{T}\sum_{t=1}^T\hat{p}_t$ where $\hat{p}_t=\Phi_t(\boldx^t,\tilde{\boldsymbol{Y}}^{t-1})$. Let $f_{j^*,\ell^*}$ be the ground truth function. Then
 \begin{align}\label{eq:klaverage}
 \mathbb{E}_{\tilde{\boldsymbol{Y}}^T}[\kl(p,\hat{p})]\le& \frac{\mathbb{E}_{\tilde{\boldsymbol{Y}}^T}[\sum^T_{t=1}\kl(f_{j^*,\ell^*}[\boldx_t],\hat{p}_t)]}{T}\nonumber\\
 =&\frac{r^{\kl}_T(\cF)}{T}.     
 \end{align}
In the following, we show that 
\begin{align}\label{eq:ekldistancelower}
\mathbb{E}_{\tilde{\boldsymbol{Y}}^T}[\kl(p,\hat{p}(\tilde{\boldsymbol{Y}}^T))]\ge \frac{a}{33*4}=\frac{a}{132}
\end{align}

for any estimator $\hat{p}:\tilde{\cY}^T\rightarrow \mathcal{D}([N])$, which proves the theorem.

We use the Le Cam's two point method. Note that for any $\hat{p}\in \mathcal{D}([N])$, 
\begin{align}\label{eq:twopoint}
   \max\{\kl(p_{i,1},\hat{p}),\kl(p_{i,2},\hat{p})\} \overset{(a)}{\ge}\frac{\kl(p_{i,1},\frac{p_{i,1}+p_{i,2}}{2})}{2}=\frac{\log (\frac{2}{2-\frac{a}{2}})}{2}\overset{(b)}{\ge} \frac{a}{16}, 
\end{align}
where $(a)$ follows by~\citet[Lemma C.1]{wu2023learning} and $(b)$ follows since $0\le a\le 1$ and $\log (1+x)\ge \frac{x}{2}$ when $x\le 1$. In the following we show that there exists $i\in[K]$ such that
$$|\tilde{p}^{\otimes T}_{i,1}-\tilde{p}^{\otimes T}_{i,2}|_{TV}\le \frac{1}{3},$$
where $\tilde{p}_{i,\ell}$ is the distribution of $\mathcal{R}_t(Y_t)$ when $Y_t\sim p_{i,\ell}$ for $i\in [K]$ and $\ell\in \{1,2\}$.



Let $\cR_t$ be described by conditional distribution $q_t(\Tilde{y}|y)$ for $y\in [N]$ and $\Tilde{y}\in \Tilde{\cY}$, where $\tilde{\cY}$ is the output space of $\cR_t$. For each $\tilde{y}\in\Tilde{\cY}$, let $\boldq(\Tilde{y})=(q(\Tilde{y}|1),\ldots,q(\Tilde{y}|N))$ be a vector and let 
\begin{align}\label{eq:qdecomposition}
\boldq(\Tilde{y})=\sum^N_{y=1}b(\tilde{y})_yH^N_{y}
\end{align}
where $b(\tilde{y})_y$ are the coefficients associated with the orthogonal base $\{H_y^N\}_{y\in \mathcal{Y}}$.
We have
\begin{align}\label{eq:tildep}
\tilde{p}_{i,\ell}(\tilde{y})=\sum^N_{y=1}q_t(\tilde{y}|y)p_{i,\ell}(y).
\end{align} 
From \eqref{eq:distributionp}, \eqref{eq:qdecomposition}, and \eqref{eq:tildep}, we know that
\begin{align*}
&\tilde{p}_{i,1}(\tilde{y})=\frac{2\sum_{y:H^N_{i+1,y}=-1}q_t(\Tilde{y}|y)}{N}, \mbox{ and}\\
&\tilde{p}_{i,2}(\tilde{y})=\frac{2\sum_{y:H^N_{i+1,y}=-1}q_t(\Tilde{y}|y)}{N}+a b(\tilde{y})_{i+1},
\end{align*}
for $i\in[N]\backslash\{1\}$.
The KL distance
\begin{align}\label{eq:KL}
\kl(\tilde{p}_{i,1},\tilde{p}_{i,2})&=\int_{\tilde{\cY}}-\tilde{p}_{i,1}(\tilde{y})\log \Big(\frac{\tilde{p}_{i,2}(\tilde{y})}{\tilde{p}_{i,1}(\tilde{y})}\Big)d\tilde{y}\nonumber\\
&\overset{(a)}{\le}\int_{\tilde{\cY}}-\tilde{p}_{i,1}(\tilde{y}) \Bigg(\frac{\tilde{p}_{i,2}(\tilde{y})-\tilde{p}_{i,1}(\tilde{y})}{\tilde{p}_{i,1}(\tilde{y})}-\Big(\frac{\tilde{p}_{i,2}(\tilde{y})-\tilde{p}_{i,1}(\tilde{y})}{\tilde{p}_{i,1}(\tilde{y})}\Big)^2\Bigg)d\tilde{y}\nonumber\\
&=\int_{\tilde{\cY}}\frac{(\tilde{p}_{i,2}(\tilde{y})-\tilde{p}_{i,1}(\tilde{y}))^2}{\tilde{p}_{i,1}(\tilde{y})}d\tilde{y}\nonumber\\
&\le\int_{\tilde{\cY}}\frac{a^2 b(\tilde{y})^2_{i+1}}{\Big(\min_{y}q_t(\Tilde{y}|y)\Big)}d\tilde{y},
\end{align}
where $(a)$ follows from the fact that $x-x^2\le\log(1+x)$ for $x\ge 0$.

One the other hand, we have that
\begin{align}\label{eq:parsevalinequality}
\sum_{y\in[N-1]}b(\tilde{y})^2_{y+1}&=\sum_{y\in[N-1]}\Big(\frac{\sum_{i\in[N]}q_t(\tilde{y}|i)H^N_{y+1,i}}{N}\Big)^2\nonumber\\
&\overset{(a)}{=}\frac{\sum_{i\in[N]}q_t(\tilde{y}|i)^2}{N}-\Big(\frac{\sum_{i\in[N]}q_t(\tilde{y}|i)}{N}\Big)^2\nonumber\\
&\overset{(b)}{\le}\frac{\sum_{i\in[N]}\Big(\sum_{j\in[K]}(q_t(\tilde{y}|i)-q_t(\tilde{y}|j))^2\Big)}{2N^2}\nonumber\\
&\overset{(c)}{\le}\frac{\min\{(e^\epsilon-1)^2,1\}\sum_{i\in[N]}q_t(\tilde{y}|i)^2}{2N}\nonumber\\
&\le\frac{\min\{(e^\epsilon-1)^2,1\}e^\epsilon\sum_{i\in[N]}\min_{y}q_t(\Tilde{y}|y)q_t(\tilde{y}|i)}{2N},
\end{align}
where $(a)$ follows from the Parseval's identity and the fact that $\{H_y^N\}_{y\in \mathcal{Y}}$ form orthogonal bases, $(b)$ follows from the elementary identity $N\sum_{i=1}^Na_n^2-(\sum_{i=1}^Na_n)^2=\frac{1}{2}\sum_{i,j\le N}(a_i-a_j)^2$, and $(c)$ follows frpm property of $(\epsilon,0)$-differential privacy. Now,
\eqref{eq:KL} and  \eqref{eq:parsevalinequality} imply
\begin{align*}
  \sum_{i\in[K]} \kl(\tilde{p}_{i,1},\tilde{p}_{i,2})&\le\int_{\tilde{\cY}}\frac{\sum_{i+1\in[N]}a^2 b(\tilde{y})^2_{i+1}}{\min_{y}q_t(\Tilde{y}|y)}d\tilde{y}\\
&\le a^2\int_{\tilde{\cY}}\frac{\min\{(e^\epsilon-1)^2,1\}e^\epsilon\sum_{i\in[N]}q_t(\tilde{y}|i)}{2N}d\tilde{y}\\
&\le\frac{a^2}{2}\min\{(e^\epsilon-1)^2,1\}e^{\epsilon}.
\end{align*}
Therefore, there exists $i^*\in[K]$ such that 
\begin{align*}
   \kl(\tilde{p}_{i^*,1},\tilde{p}_{i^*,2})\le \frac{a^2\min\{(e^\epsilon-1)^2,1\}e^{\epsilon}}{2K},
\end{align*}
which implies that 
\begin{align}\label{eq:tvconstantupper}
    |\tilde{p}^{\otimes T}_{i^*,1}-\tilde{p}^{\otimes T}_{i^*,2}|_{TV}\le &\sqrt{2\kl(\tilde{p}^{\otimes T}_{i^*,1},\tilde{p}^{\otimes T}_{i^*,2})}\nonumber\\
    \le& \sqrt{2T\kl(\tilde{p}_{i^*,1},\tilde{p}_{i^*,2})}\nonumber\\
    <&\frac{1}{3}.
\end{align}
The rest of the argument is standard. Let $p$ be one of $p_{i,1}$ and $p_{i,2}$ and let $\phi:\tilde{\cY}^T\rightarrow \{p_{i,1},p_{i,2}\}$ be a classification function deciding if $p$ is $p_{i,1}$ or $p_{i,2}$ based on the privatized data $\tilde{\boldsymbol{Y}}^T$. From \eqref{eq:tvconstantupper} and   Le Cam's two point lemma, the classification error
\begin{align}\label{eq:classificationerror}
\mathrm{Pr}(\phi(\tilde{\boldsymbol{Y}}^T)\ne p_{i,\ell^*})\ge\frac{1-|\tilde{p}^{\otimes T}_{i^*,1}-\tilde{p}^{\otimes T}_{i^*,2}|_{TV}}{2}\ge \frac{1}{3},
\end{align}
where $p=p_{i,\ell^*}$. 

On the other hand, suppose there exists a private hypothesis testing estimator $\hat{p}$ satisfying $\kl(p_{i,\ell^*},\hat{p}(\tilde{\boldsymbol{Y}}^T))\le \frac{a}{33}$ with probability at least $\frac{3}{4}$. Then, from \eqref{eq:twopoint} we conclude that a minimum $\kl$ distance classifier $\phi(\tilde{\boldsymbol{Y}}^T)=\mbox{argmin}_{p^*\in\{p_{i^*,1},p_{i^*,2}\}}\kl(p^*,\hat{p}(\tilde{\boldsymbol{Y}}^T))$ is correct with probability at least $\frac{3}{4}$, contradicting \eqref{eq:classificationerror}. Therefore, we have \eqref{eq:ekldistancelower}, and thus the theorem.

\section{Proof of Lemma~\ref{lem:wmbound}}
\label{ap:prooflem5}
Let $\cE$ denote the event that there exists $j\in[K]$, $t\in [T]$ such that 
\begin{align}\label{eq:Laplacianbound}
|Lap^t_j|\ge c'(\gamma),
\end{align}
where $c'(\gamma)$ is given in \eqref{eq:cgamma}. 
Note that for each $t\in[T]$ and $j\in[K]$, \eqref{eq:Laplacianbound} occurs with probability $e^{-\gamma-\log K-\log T}$, $t\in[T]$, $j\in[K]$. Hence, by the union bound, the probability of $\cE$ is at most $e^{-\gamma}$. 

In the following, we show that Lemma~\ref{lem:wmbound} holds when $\cE$ does not occur, which has probability at least $1-e^{-\gamma}$. 

Note that from \eqref{eq:clipping}, we have $\log (\bar{p}_{t,j}[h_t(Y_t)])\in [\log (\frac{1}{TKM}),\log (\frac{1}{M})]$. Hence, 
\begin{align*}
    \log \big(\bar{p}_{t,j}[h_t(Y_t)]\big)+Lap^t_j+\log M-c'(\gamma)\in [-\log (KT)-2c'(\gamma),0], 
\end{align*}
and thus that $\tilde{Y}_{t,j}\in[0,1]$.  
According to Lemma \ref{lem:regretbound}, we have 
\begin{align*}
\sum^T_{t=1}\sum^K_{j=1}\Tilde{w}^t_j\Big(-\log \big(\bar{p}_{t,j}[h_t(Y_t)]\big)-Lap^t_j\Big)-\min_{i\in[K]}\Bigg(\sum^T_{t=1}\Big(-\log \big(\bar{p}_{t,i}[h_t(Y_t)]\big)-Lap^t_i\Big) \Bigg)
\le \frac{1}{c(\gamma)}\sqrt{2T\log K}.
\end{align*}
Let $j^*\in [K]$ be any index, and assume that $Y_t\sim f_{j^*}(\x_t)$. Taking expectation over the distribution of $Y'_t\sim h_t(Y_t)$ we find that
\begin{align}\label{eq:lossupperbound}
\mathbb{E}_{Y'^T}\left[\sum^T_{t=1}\sum^K_{j=1}\Tilde{w}^t_j\kl(\bar{p}_{t,j^*},\bar{p}_{t,j})\right]
\le \frac{1}{c(\gamma)}\sqrt{2T\log K}+\sum^T_{t=1}\sum^K_{j=1}\mathbb{E}_{Y'^T}[\Tilde{w}^t_j]Lap^t_j-\sum^T_{t=1}Lap^t_{j^*},
\end{align}
where we used the fact that $Y'_t$ distributed as $\bar{p}_{t,j^*}=h_t\circ f_{j^*}(\x_t)$ and $\mathbb{E}_{Y\sim p}\log\left(\frac{p[Y]}{q[Y]}\right)=\kl(p,q)$.
Moreover, we have by Jensen's inequality and convexity of $-\log(x)$ that
\begin{align*}
\sum^T_{t=1}\Big(-\log \Big(\sum^K_{j=1}\Tilde{w}^t_j\bar{p}_{t,j}[h_t(Y_t)]\Big)\Big)-&\sum^T_{t=1}\Big(-\log \big(\bar{p}_{t,j^*}[h_t(Y_t)]\big)\Big)\nonumber\\
&\le
\sum^T_{t=1}\sum^K_{j=1}\Tilde{w}^t_j\Big(-\log \big(\bar{p}_{t,j}[h_t(Y_t)]\big)\Big)-\sum^T_{t=1}\Big(-\log \big(\bar{p}_{t,j^*}[h_t(Y_t)]\big)\Big).
\end{align*}
Hence, by taking expectation over the distribution of $Y'_t\sim h_t(Y_t)$
\begin{align}\label{eq:klleweightedkl}
\mathbb{E}_{Y'^T}\left[\sum^T_{t=1}\kl(\bar{p}_{t,j^*},\Bar{p}_t)\right]\le \mathbb{E}_{Y'^T}\left[\sum^T_{t=1}\sum^K_{j=1}\Tilde{w}^t_j\kl(\bar{p}_{t,j^*},\bar{p}_{t,j})\right].
\end{align}
Combining \eqref{eq:lossupperbound} and \eqref{eq:klleweightedkl}, we arrive at
\begin{align}\label{eq:klbarppupper}   \mathbb{E}_{Y'^T}\left[\sum^T_{t=1}\kl(\bar{p}_{t,j^*},\Bar{p}_t)\right]
    \le \frac{1}{c(\gamma)}\sqrt{2T\log K}+\sum^T_{t=1}\sum^K_{j=1}\mathbb{E}_{Y'^T}[\Tilde{w}^t_j]Lap^t_j-\sum^T_{t=1}Lap^t_{j^*}.
\end{align}
On the other hand, 
\begin{align}\label{eq:klhp}
    \kl(f_{j^*}(\mathbf{x}_t),\hat{p}_t)
    &=\sum_{y\in\mathcal{Y}}f_{j^*}(\mathbf{x}_t)[y]\log \left(\frac{f_{j^*}(\mathbf{x}_t)[y]}{\hat{p}_t[y]}\right)\nonumber\\
    &\overset{(a)}{=}\sum_{y\in\mathcal{Y}}\left(\sum_{y'\in\mathcal{S}_{y,t}}h'_t\circ f_{j^*}(\mathbf{x}_t)[y']\right)\cdot\log \left(\frac{\sum_{y'\in\mathcal{S}_{y,t}}h'_t\circ f_{j^*}(\mathbf{x}_t)[y']}{\sum_{y'\in\mathcal{S}_{y,t}}\left(\sum_{j\in[K]}\tilde{w}^t_j h_t\circ f_{j}(\mathbf{x}_t)[y']\right)}\right)\nonumber\\
    &\overset{(b)}{=}\sum_{y\in\mathcal{Y}}\sum_{y'\in\mathcal{S}_{y,t}}\left(h'_t\circ f_{j^*}(\mathbf{x}_t)[y']\cdot\log \left(\frac{h'_t\circ f_{j^*}(\mathbf{x}_t)[y']}{\sum_{j\in[K]}\tilde{w}^t_j h_t\circ f_{j}(\mathbf{x}_t)[y']}\right)\right)\nonumber\\
    &=\kl(h'_t\circ f_{j^*}(\mathbf{x}_t),\bar{p}_t),
\end{align}

where $(a)$ follows from the definition of $h'_t$ and the fact that $\hat{p}_t[y]=\sum_{y'\in \mathcal{S}_y}\bar{p}_t[y]$ where $\bar{p}_{t}=\sum_{j\in [K]}\tilde{w}_j^th_t\circ f_j(\boldx_t)$ (see Algorithm~\ref{alg:WM}), and $(b)$ follows from the fact that $h'_t\circ f_{j^*}(\boldx_t)[y']$ and $h_t\circ f_{j^*}(\boldx_t)[y']$ are \emph{constants} for all $y'\in \cS_{y,t}$. 
Moreover, 
\begin{align}\label{eq:klhprimebarpt}
  &\kl(h'_t\circ f_{j^*}(\boldx_t),\Bar{p}_t)\nonumber\\
  &=\sum_{y'\in[N'_t]}h'_t\circ f_{j^*}(\boldx_t)[y']\log \Big(h'_t\circ f_{j^*}(\boldx_t)[y']\Big) - \sum_{y'\in[N'_t]}h'_t\circ f_{j^*}(\boldx_t)[y']\log \Big(\Bar{p}_t[y']\Big)\nonumber\\
  &=\sum_{y'\in[N'_t]}h'_t\circ f_{j^*}(\boldx_t)[y']\log \Bigg(\Big(1-\frac{1}{T}\Big)h'_t\circ f_{j^*}(\boldx_t)[y']\Bigg) - \sum_{y'\in[N'_t]}h'_t\circ f_{j^*}(\boldx_t)[y']\log \Big(\Bar{p}_t[y']\Big) - \log \Big(1-\frac{1}{T}\Big)\nonumber\\
  &\le \sum_{y'\in[N'_t]}h'_t\circ f_{j^*}(\boldx_t)[y']\cdot\log \Bigg(\Big(1-\frac{1}{T}\Big)h'_t\circ f_{j^*}(\boldx_t)[y'] + \frac{1}{TN'_t}\Bigg)-\sum_{y'\in[N'_t]}h'_t\circ f_{j^*}(\boldx_t)[y']\log \Bar{p}_t[y'] - \log \Big(1-\frac{1}{T}\Big)\nonumber\\
  &=\sum_{y'\in[N'_t]}\Bigg(\Big(1-\frac{1}{T}\Big)h'_t\circ f_{j^*}(\boldx_t)[y'] + \frac{1}{TN'_t}\Bigg)\cdot\log \Bigg(\frac{\Big(1-\frac{1}{T}\Big)h'_t\circ f_{j^*}(\boldx_t)[y'] + \frac{1}{TN'_t}}{\Bar{p}_t[y']}\Bigg) \nonumber\\
  &\qquad+\sum_{y'\in[N'_t]}\Bigg(\frac{1}{T}h'_t\circ f_{j^*}(\boldx_t)[y'] - \frac{1}{TN'_t}\Bigg)\cdot\log \Bigg(\frac{\Big(1-\frac{1}{T}\Big)h'_t\circ f_{j^*}(\boldx_t)[y'] + \frac{1}{TN'_t}}{\Bar{p}_t[y']}\Bigg) - \log \Big(1-\frac{1}{T}\Big)\nonumber\\
  &=\kl(\bar{p}_{t,j^*},\bar{p}_t) + \sum_{y'\in[N'_t]}\Bigg(\frac{1}{T}h'_t\circ f_{j^*}(\boldx_t)[y'] - \frac{1}{TN'_t}\Bigg)\cdot\log \Bigg(\frac{\bar{p}_{j^*}[y']}{\Bar{p}_t[y']}\Bigg) - \log \Big(1-\frac{1}{T}\Big)\nonumber\\
  &\overset{(a)}{\le}\kl(\bar{p}_{t,j^*},\bar{p}_t) + \frac{\log (KT)}{T} + \frac{\log (KT)}{KTM} - \log \Big(1-\frac{1}{T}\Big)
\end{align}

where $(a)$ follows because of \eqref{eq:clipping} and the fact that $\bar{p}_t$ is a linear combination of $\bar{p}_j$, $j\in[K]$.
Combining \eqref{eq:klbarppupper}, \eqref{eq:klhp}, and \eqref{eq:klhprimebarpt}, we have
\begin{align*}
    \mathbb{E}_{Y'^T}\left[\sum^T_{t=1}\kl(f_{j^*}(\boldx_t),\hat{p}_t)\right]&=\mathbb{E}_{Y'^T}\left[\sum^T_{t=1} \kl(h'_t\circ f_{j^*}(\boldx_t),\Bar{p}_t)\right]\\
    &\le\mathbb{E}_{Y'^T}\left[\sum^T_{t=1}\kl(\bar{p}_{t,j^*},\bar{p}_t)-T\log (1-\frac{1}{T})+2\log (KT)\right]\\
    &\le \frac{1}{c(\gamma)}\sqrt{2T\log K}+3\log (KT)+\sum^T_{t=1}\sum^K_{j=1}\mathbb{E}_{Y'^T}[\Tilde{w}^t_j]Lap^t_j-\sum^T_{t=1}Lap^t_{j^*}.
\end{align*}
This completes the proof.

\section{Proof of Lemma~\ref{lem:wmboundpdp}}
\label{app:prooflem8}
The proof follows the footsteps of the proof of Lemma \ref{lem:wmbound} by using Lemma \ref{lem:regretboundpdp}. 
Let $\cE$ denote the event that there exists $j\in[K]$, $t\in [T]$ such that 
\begin{align}\label{eq:Laplacianboundpdp}
|Lap^t_j|\ge c_{pdp}'(\gamma).
\end{align}
Then, the probability of $\cE$ is at most $e^{-\gamma}$. 
We show that Lemma~\ref{lem:wmboundpdp} holds when $\cE$ does not occur. 
Note that from \eqref{eq:clipping}, we have $\log (\bar{p}_{t,j}[h_t(Y_t)])\in [\log (\frac{1}{TKM}),\log (\frac{1}{M})]$. Therefore, 
\begin{align*}
    \log \big(\bar{p}_{t,j}[h_t(Y_t)]\big)+Lap^t_j+\log M-c'(\gamma)\in [-\log (KT)-2c'(\gamma),0], 
\end{align*}
and thus that $\mathbf{Z}_{t,j}\in[0,1]^K$ for all $t\in [T]$ and $j\in [K]$.

Invoking Lemma \ref{lem:regretboundpdp}, we have
\begin{align}\label{eq:lossupperboundpdp}
\max_{i\in [K]}\mathbb{E}_{J^T}\left[\sum^T_{t=1}\left(\sum^K_{j=1}\Tilde{w}^t_jZ_{t,j,j}-\sum^T_{t=1}Z_{t,i,i}\right)\right]&= \max_{i\in [K]}\mathbb{E}_{J^T}\Bigg[\sum^T_{t=1}\Big(\sum^K_{j=1}\Tilde{w}^t_j(-c_{pdp}(\gamma)(\log \big(\bar{p}_{t,j}[h_t(Y_t)]\big)+Lap^t_j))\nonumber\\
&\qquad -c_{pdp}(\gamma)(\log \big(\bar{p}_{t,j}[h_t(Y_t)]\big)+Lap^t_j)\Big)\Bigg]\nonumber\\
&\le \sqrt{2TK\log K}.
\end{align}
Let $j^*\in [K]$ be any index and assume $Y_t\sim f_{j^*}(\x_t)$. Taking expectations over the distributions of $Y'_t\sim h_t\circ f_{j^*}(\x_t)$ for $t\in[T]$, we find (similar to~\eqref{eq:lossupperbound})
\begin{align}\label{eq:kllossbandit}
\mathbb{E}_{Y'^T,J^T}\left[\sum^T_{t=1}\sum^K_{j=1}\Tilde{w}^t_j\kl(\bar{p}_{t,j^*},\bar{p}_{t,j})\right]\le \frac{1}{c_{pdp}(\gamma)}\sqrt{2TK\log K} + \sum^T_{t=1}\sum^K_{j=1}\mathbb{E}_{Y'^T,J^T}\left[\Tilde{w}^t_j\right]Lap^t_j-\sum^T_{t=1}Lap^t_{j^*}.
\end{align}
Combining \eqref{eq:klleweightedkl}, \eqref{eq:klhp},  \eqref{eq:klhprimebarpt}, and \eqref{eq:kllossbandit}, we conclude
\begin{align*}
    \mathbb{E}_{Y'^T,J^T}\left[\sum^T_{t=1}\kl(f_{j^*}(\boldx_t),\hat{p}_t)\right]\le \frac{1}{c_{pdp}(\gamma)}\sqrt{2TK\log K}+3\log(KT)+\sum^T_{t=1}\sum^K_{j=1}\mathbb{E}_{Y'^T,J^T}\left[\Tilde{w}^t_j\right]Lap^t_j-\sum^T_{t=1}Lap^t_{j^*}.
\end{align*}
This completes the proof.

\vfill
\end{document}